\documentclass{article}

     \usepackage[final]{neurips_2019}

\usepackage[utf8]{inputenc} 
\usepackage[T1]{fontenc}    
\usepackage{hyperref}       

\usepackage{url}            
\usepackage{booktabs}       
\usepackage{amsfonts}       
\usepackage{nicefrac}       
\usepackage{microtype}      
\usepackage{array}
\usepackage{amsthm}

\theoremstyle{theorem}
\newtheorem{proposition}{Proposition}

\title{ALPaCA vs. GP-based Prior Learning: \\A Comparison between two Bayesian Meta-Learning Algorithms}

\author{%
  Yilun Wu \\
  D-MAVT\\
  ETH Zürich\\
  Switzerland, ZH 8001 \\
  \texttt{wuyil@student.ethz.ch} \\
}

\begin{document}

\maketitle

\begin{abstract}
  Meta-learning or few-shot learning, has been successfully applied in a wide range of domains from computer vision to reinforcement learning. Among the many frameworks proposed for meta-learning, bayesian methods are particularly favoured when accurate and calibrated uncertainty estimate is required. In this paper, we investigate the similarities and disparities among two recently published bayesian meta-learning methods: ALPaCA (\cite{alpaca}) and PACOH (\cite{pacoh}). We provide theoretical analysis as well as empirical benchmarks across synthetic and real-world dataset. While ALPaCA holds advantage in computation time by the usage of a linear kernel, general GP-based methods provide much more flexibility and achieves better result across datasets when using a common kernel such as SE (Squared Exponential) kernel. The influence of different loss function choice is also discussed.
\end{abstract}

\section{Introduction}
Many sources such as \cite{Mnih2015HumanlevelCT} and \cite{dubey2018investigating} suggest what make humans so good at learning to solve new tasks is by efficiently extracting prior knowledge from past experience and leveraging them during decision making. 

Acquiring such generalizable inductive bias has always been a central problem in machine learning (\cite{Baxter_2000}), especially under meta-learning settings where the amount of online learning data is limited (\cite{caruana1997multitask}, \cite{vanschoren2018metalearning}). However, adapting a highly complex model based on a few training input is inherently challenging and produces a large uncertainty if not well regulated (e.g. meta-overfitting reported by \cite{mishra2017simple} and \cite{pacoh}). To improve the robustness of meta-learning methods, it is therefore a good idea to incorporate uncertainty statistics as a metric during training and evaluation in the learning process. 

Bayesian methods are very popular in this regime, since it provides meaningful and accurate statistics through well-understood update rule. Both parametric methods such as Bayesian Neural Network and Non-parametric methods such as GP (Gaussian Processes) have seen tremendous successful real-world applications such as \cite{berkenkamp2017safe}, \cite{Hewing_2019} and \cite{2017_Peretroukhin_Reducing}. The representation power of standard GP methods are furthermore extended by embedding function approximators such as deep neural networks in its kernel (\cite{wilson2015deep}) and mean function (\cite{fortuin2019}) or jointly (\cite{Calandra_2016} and \cite{garnelo2018conditional}). It is therefore natural to adopt bayesian statistics in the scope of meta-learning where bayesian priors are used as the inductive bias at the task-level. Recent adoptions in  meta-learning literatures include \cite{kim2018bayesian}, \cite{grant2018recasting}, \cite{ravi2018amortized}.

In this work, we focus on comparing two similar bayesian learning frameworks: BLR (Bayesian Linear Regression) vs. GPR (Gaussian Process Regression) as published in \cite{alpaca} and \cite{pacoh}. We show an one-to-one correspondence between BLR models and GPR models in \autoref{sec:equivalence}, in the sense that BLR and GPR models share the same hypothesis space. We later show the objective function used in \cite{alpaca}, as a result of directly minimizing the conditional (posterior) KL divergence between the true distribution and predicted distribution, under mild assumptions, is the same as maximizing the likelihood of overall dataset. In \autoref{sec:empirical}, we present empirical results on models consisting of different architectures and loss function with varying nuances. A trade-off between performance and computation still exists between kernel-based GP methods and bayesian linear regression, despite opposite claims from \cite{alpaca}.

\section{Theoratical Equivalence and Disparity between ALPaCA and GPR Meta Learning}
In this section, we point out the theoretical equivalence and differences between ALPaCA \cite{alpaca} and PACOH \cite{pacoh} in the meta learning setting. 

We first show that during inference (online update), the posterior predictive distribution of response variable conditioned on context data from BLR (Bayesian Linear Regression) in latent space is the same as the posterior predictive distribution derived from a GPR (Gaussian Process Regression) with a special linear kernel.

We then compare the loss function used for learning the parameters constituting the meta prior in two methods. It can be shown that under the assumption of uniform sampling of context time horizon, the two loss functions are equivalent to each other.

\subsection{Latent Space Bayesian Linear Regression}
\label{sec:ALPaCa-math}
We summarize the meta-learning architecture used in ALPaCA as bayesian linear regression with
transformed feature space via basis function $\phi:\mathbb{R}^{n_x}\rightarrow\mathbb{R}^{n_\phi}$. This is usually implemented with a deep neural network with $n_\phi$ hidden layers in the last layer, as in \cite{alpaca}. The dependent variable y is linearly regressed as: $y^\top=\phi(x)^\top\mathbf{K}+\epsilon$, where $\epsilon\sim\mathcal{N}(0, \mathbf{\Sigma_\epsilon}), \mathbf{K}\in\mathbb{R}^{n_\phi\times n_y}$.
Note, in this case, both $x, \phi(x)$ and $y$ are multi-dimensional vectors with dimension $n_x, n_\phi, n_y$ respectively.

Assume a prior for $\mathbf{K}\sim \mathcal{MN}(\mathbf{K_0}, \Lambda_0^{-1}, \Sigma_\epsilon )$ where $\mathcal{MN}$ indicates matrix normal distribution.\footnote{The definition of the matrix normal distribution can be found at \url{https://en.wikipedia.org/wiki/Matrix_normal_distribution}.} With the likelihood (conditional probability) of data $Y$ being $p(Y|\Phi, \mathbf{K}, \Sigma_\epsilon)\propto \abs{\Sigma_\epsilon}^{-n_y/2} \exp(-\frac{1}{2}\tr((Y-\Phi\mathbf{K})\Sigma_\epsilon^{-1}(Y-\Phi\mathbf{K})^\top))$, where $\Phi=[\phi(x_1), \phi(x_2), ...]$, we arrive at the posterior distribution of $\mathbf{K}$: $p(\mathbf{K}|Y,X)=\mathcal{MN}(\mathbf{K_\tau}, \Lambda_\tau^{-1}, \Sigma_\epsilon)$, which is of the same distribution family as the prior, making matrix normal distribution the conjugate prior for the likelihood function above.

Plugging the posterior distribution of $\mathbf{K}$ into the likelihood function, we obtain posterior predictive distribution of $Y_t$ at test points $X_t$ conditioned on observed context data $(X_c, Y_c)$:

\begin{equation}
p(Y_t|X_t, X_c, Y_c)=\mathcal{N}(\Phi(X_t)^\top\mathbf{K_\tau}, \Sigma_\tau)
\label{eqn:alpaca-post}
\end{equation}
where
\begin{align}
\label{eqn:BLR-inv}
\mathbf{K_\tau}&=\Lambda_\tau^{-1}(\Phi(X_c) Y_c + \Lambda_0\mathbf{K_0}) \\
\label{eqn:lambda-tau}
\Lambda_\tau &=\Phi(X_c) {\Phi(X_c)}^\top+\Lambda_0 \\
\Sigma_{\tau}&=(\mathbf{I}+\Phi(X_t)^\top \Lambda_\tau^{-1} \Phi(X_t))\otimes \Sigma_\epsilon
\end{align}
where $\otimes$ denotes the Kronecker product.\footnote{The definition of Kronecker product can be found at \url{https://en.wikipedia.org/wiki/Kronecker_product}.}
\subsection{Gaussian Process Regression with Linear Kernel}
Now consider multi-output (multi-task) Gaussian Process Regression $f: X\in\mathbb{R}^{n_x}\rightarrow Y\in\mathbb{R}^{n_y} \sim \mathcal{GP}(m(\cdot), k(\cdot, \cdot))$ where $m(\cdot): X\in\mathbb{R}^{n_x}\rightarrow Y\in\mathbb{R}^{n_y}$ is the mean function and $k(\cdot, \cdot): X \times X\in\mathbb{R}^{n_x} \times \mathbb{R}^{n_x} \rightarrow \Sigma \in\mathbb{R}^{n_y \times n_y}$ is the kernel function. Note in this case, the output of $f$ is multi-dimensional, thus the output of the kernel function is a matrix instead of a scalar value. The matrix output of the kernel function can be interpreted as the cross-covariance matrix at two different intputs under their prior distribution: $k(x, x')=\mathbb{E}[(f(x)-\overline{f(x)})(f(x')-\overline{f(x')})^\top]$.

Let $m(x)=\mathbf{K_0}^\top \phi(x)$ and $k(x, x')=\phi(x)^\top \Lambda_0^{-1} \phi(x') \Sigma_\epsilon$ with the same $\mathbf{K_0}, \Lambda_0, \Sigma_\epsilon$ used in defining the prior distribution of $\mathbf{K}$ in \ref{sec:ALPaCa-math}. One can construct the gram matrix of a set of points $X=[x_1, x_2, ..., x_n]^\top$ by stacking pairs of kernel functions:
\begin{align}
	\label{eq:gram-matrix}
	G(X)&=
	\begin{bmatrix}
		k(x_1, x_1) & k(x_1, x_2) & \dots & k(x_1, x_n) \\
		k(x_2, x_1) & k(x_2, x_2) & \dots & k(x_2, x_n)\\
		\vdots & \vdots & \ddots & \vdots \\
		k(x_n, x_1) & k(x_n, x_2) & \dots & k(x_n, x_n)
	\end{bmatrix} \\
	&=\Phi(X)^\top \Lambda_0^{-1}\Phi(X) \otimes \Sigma_\epsilon
\end{align}

Further let $y(x)=f(x)+\epsilon$ where $\epsilon\sim\mathcal{N}(0, \Sigma_\epsilon)$. Then the prior distribution for $Y(X)=[y(x_1), y(x_2), ..., y(x_n)]^\top$ reads as follows:
\begin{equation}
	\label{eqn:gp-likelihood}
	Y\sim \mathcal{N}(m(X), (G(X)+\mathbf{I})\otimes \Sigma_\epsilon)
\end{equation}

From this we can write the following prior distribution for target values at context points $X_c$ and test points $X_t$:
\begin{equation}
	\begin{bmatrix}
		Y(X_c)\\
		Y(X_t)
	\end{bmatrix} \sim \mathcal{N}\left(
	\begin{bmatrix}
		\mu_1 \\ \mu_2
	\end{bmatrix}
,	\begin{bmatrix}
		\Sigma_{11} & \Sigma_{12} \\
		\Sigma_{21} & \Sigma_{22}
	\end{bmatrix}
\right)
\end{equation} 
where 
\begin{align}
	\mu_1= \Phi(X_c)^\top\mathbf{K_0},\;& \mu_2=\Phi(X_t)^\top \mathbf{K_0} \\
	\Sigma_{11}=(\Phi(X_c)^\top\Lambda_0^{-1}\Phi(X_c)+\mathbf{I})\otimes \Sigma_\epsilon , \; & \Sigma_{22}=(\Phi(X_t)^\top\Lambda_0^{-1}\Phi(X_t)+\mathbf{I})\otimes \Sigma_\epsilon \\
	\Sigma_{12} = \Sigma_{21}^\top=&(\Phi(X_c)^\top\Lambda_0^{-1}\Phi(X_t))\otimes \Sigma_\epsilon
\end{align}
By applying marginalization rule, we arrive at the following posterior distribution for $Y_t$ at test points $X_t$:
\begin{align}
	\label{eqn:gp-posterior}
	&p(Y_t|X_t, X_c, Y_c=z)= \mathcal{N}(\mu_2+\Sigma_{21}\Sigma_{11}^{-1}(z-\mu_1), \Sigma_{22}-\Sigma_{21} \Sigma_{11}^{-1} \Sigma_{12}) \\
	=&\mathcal{N}\{\Phi(X_t)^\top \mathbf{K_0}+\Phi(X_t)^\top \Lambda_0^{-1}\Phi(X_c)(\Phi(X_c)^\top\Lambda_0^{-1}\Phi(X_c)+\mathbf{I})^{-1}(Y_c-\Phi(X_c)^\top \mathbf{K_0}), \nonumber\\
	&[\Phi(X_t)^\top\Lambda_0^{-1}\Phi(X_t)+\mathbf{I}-\Phi(X_t)^\top\Lambda_0^{-1}\Phi(X_c)(\Phi(X_c)^\top\Lambda_0^{-1}\Phi(X_c)+\mathbf{I})^{-1}(\Phi(X_c)^\top\Lambda_0^{-1}\Phi(X_t))]\nonumber \\
	&\otimes \Sigma_\epsilon \}
	\label{eqn:gpr-post}
\end{align}

\subsection{Equivalence between GPR with Linear Kernel and Latent Space BLR}
\label{sec:equivalence}
\begin{proposition}
\label{prop:arch-equivalence}
	The posterior predictive distribution of BLR (\autoref{eqn:alpaca-post}) and GPR (\autoref{eqn:gpr-post}) evaluates to the same distribution. 
\end{proposition}

\begin{proof}
	By invoking the Woodbury identity, we show the detailed proof in \autoref{sec:proof}.
\end{proof}

Note although they both lead to the same posterior distribution, BLR has significant computational advantage over GPR during inference when the number of conditioned context data is large, as the dimension of matrix which needs to be inversed in \autoref{eqn:gp-posterior} is $n_x \times n_x$ while for BLR in \autoref{eqn:BLR-inv} is $n_\phi \times n_\phi$, irrespective of the number of context points.

It is also worth mentioning the redundant parametrization of $\mathbf{K_0}$ and $\Lambda_0$ in Latent BLR:

\begin{proposition}
	Latent BLR defined by prior parameters $\mathbf{K_0}, \Lambda_0$ and $\phi(x)$ leads to the same posterior predictive distribution (\autoref{eqn:alpaca-post}) as Latent BLR defined by $\mathbf{K_0^\prime}, \Lambda_0^\prime$, and $\phi'(x)$ given by
	\begin{align}
		\mathbf{K_0^\prime}&=\mathbf{L}^{\prime\top}\mathbf{K_0} \\
		\Lambda_0^{\prime -1}&=\mathbf{L}^{\prime \top} \Lambda_0^{-1} \mathbf{L'} \\
		\phi^{\prime}(x)&=\mathbf{L}^{\prime -1}\phi(x)
	\end{align}
	where $\mathbf{L'}$ is any invertible matrix.
\end{proposition}

\begin{proof}
	See \autoref{sec:proof-blr}.	
\end{proof}

Therefore,  $\mathbf{K_0}$ and $\Lambda_0$ are redundant in that one of the parameter could be fixed during training without an impact on the overall loss. Fixing $\Lambda_0$ during training has additionally led to better loss landscape and avoided numerical issues especially when used with a weight decay optimizer, which will be discussed in later sections.

Due to this equivalence, we can take $\mathbf{L'}=\mathbf{L}^{-1}$ which results in a $\Lambda_0^\prime=\Lambda_0^{\prime -1}=\mathbf{I}$, rendering its equivalent GPR kernel a standard linear kernel $k'(x, x')=\phi'(x)^\top \phi'(x)$ instead of a weighted one. 

\subsection{Extended Equivalence to Mercel Kernel GPRs}
\label{sec:kernel-equivalence}
It is a classical result in machine learning that Bayesian Linear Regression can be regarded as a special case of Gaussian Process Regression with which linear kernel adopted. We have shown in \autoref{sec:equivalence} the two methods yield the same posterior distribution even for multi-variate case. In this section, we extend the argument of equivalence to all GPR models with any Mercer kernels.

Mercer kernels is the family of kernels whose gram matrix for any set of inputs (\autoref{eq:gram-matrix}) is always symmetric positive definite. Since all Mercer kernels can be written as the inner product in feature space: $k(x, x')=\phi(x)^\top \phi(x')$ \cite{murphy2013machine}, we can show any GPR with a Mercer kernel and a linear mean function: $m(x)=\mathbf{K}^\top \phi(x)$ is equivalent to a BLR in the latent space defined by the feature transformation: $\phi(x)$ with $\mathbf{K_0}=\mathbf{K}$, and $\Lambda_0=\mathbf{I}$.

Although not all kernels belong to the Mercer kernel family (such as the sigmoid kernel defined as $k(x, x')=\text{tanh}(\gamma x^\top x'+r)$), many of the commonly used ones such as Gaussian kernel (RBF Kernel), linear kernel, Matern kernel are Mercer (\cite{learnwithkernels}). One caveat is that some kernels such as Gaussian kernel could have an infinite-dimensional corresponding feature vector $\phi(x)$ although its representation power could be approximated by a kernel with a large enough finite-dimensional feature vector. We show the empirical results in later sections.

\subsection{Loss Function Equivalence}
\label{sec:loss-equivalence}
Apart from the architectural difference discussed above, the other disparity between ALPaCA and GPR Meta Learning outlined in \cite{alpaca} and \cite{pacoh} is the different cost functions which are used to optimize for prior parameters. ALPaCA assumes new data flow in a streaming fashion in that we observe one pair $(x_t, y_t)$ before we observe the next pair $(x_{t+1}, y_{t+1})$. In this section, we show that considering this "per time-step" loss is equivalent to maximizing the overall likelihood of the meta dataset when assuming a uniform time distribution even under non-i.i.d. assumption of the underlying data.

\begin{proposition}
	\label{prop:loss-equvilance}
	For all probabilistic inference algorithms, the minimization of negative expected posterior likelihood (\autoref{eqn:alpaca-cost} used by ALPaCA):
	\begin{align}
	\label{eqn:alpaca-cost}
	\min_{\xi} \mathbb{E}_{t\sim p(t), \theta^*\sim p(\theta)}[\mathbb{E}_{x_{t+1}, y_{t+1}, D_t^* \sim p(x_{t+1}, y_{t+1}, D_t^* | \theta^*) } \log q_\xi(y_{t+1}|x_{t+1}, D_t^*)]
\end{align}
	 and expected negative prior likelihood over the entire horizon (\autoref{eqn:mll-cost} used by PACOH-MAP):
	\begin{align}
	\label{eqn:mll-cost}
	\min_\xi \mathbb{E}_{\theta^*\sim p(\theta)}[\mathbb{E}_{D_T^*\sim p(D^*|\theta^*)} \log l_\xi(D_T^*)]
	\end{align}
	  is equivalent, assuming a uniform distribution on the context data horizon $t$.
\end{proposition}

\begin{proof}
ALPaCA optimizes for the expected posterior likelihood as follows:
\begin{align}
	\min_{\xi=(\mathbf{K_0}, \Lambda_0, \omega)} \mathbb{E}_{t\sim p(t), \theta^*\sim p(\theta)}[\mathbb{E}_{x_{t+1}, y_{t+1}, D_t^* \sim p(x_{t+1}, y_{t+1}, D_t^* | \theta^*) } \log q_\xi(y_{t+1}|x_{t+1}, D_t^*)]
\end{align}
where $q_\xi$ refers to the conditional probability distribution defined in \autoref{eqn:alpaca-post}.

GPR MLL (Maximum Likelihood) directly optimizes for the likelihood of data over a horizon of time $T$ across tasks without explicitly considering the distribution over time step:
\begin{align}
	\min_\xi \mathbb{E}_{\theta^*\sim p(\theta)}[\mathbb{E}_{D_T^*\sim p(D^*|\theta^*)} \log l_\xi(D_T^*)]
\end{align}
where $l_\xi$ refers to the likelihood probability distribution defined in \autoref{eqn:gp-likelihood}.

Assuming uniform distribution on $p(t)$ for \autoref{eqn:alpaca-cost} and a time horizon of $T$, it becomes:
\begin{alignat}{2}
	&\min_\xi \mathbb{E}_{\theta^*\sim p(\theta)}\frac{1}{T}&&\sum_{t=0}^{T-1} \mathbb{E}_{x_{t+1}, y_{t+1}, D_t^* \sim p(x, y, D_t^* | \theta^*) } \log q_\xi(y_{t+1}|x_{t+1}, D_t^*) \\
	\label{eqn:cost-expand}
	=&\frac{1}{T} \min_\xi \mathbb{E}_{\theta^*\sim p(\theta)} [&&\mathbb{E}_{x_{1}, y_{1} \sim p(x, y| \theta^*) } \log q_\xi(y_{1}|x_{1}) +  \nonumber \\ 
	& &&\mathbb{E}_{x_{2}, y_{2}, x_{1}, y_{1} \sim p(x_2, y_2, D_1^*=(x_1, y_1)| \theta^*) } \log q_\xi(y_{2}|x_{2}, x_{1}, y_{1}) + \nonumber \\
	& &&\mathbb{E}_{x_{3}, y{3}, x_{2}, y_{2}, x_{1}, y_{1} \sim p(x_3, y_3, D_2^*=(x_1, y_1, x_2, y_2)| \theta^*) } \log q_\xi(y_{3}|x_{3}, x_{1}, y_{1}, x_{2}, y_{2}) + \nonumber \\
	& && \dots \dots ]
\end{alignat}
Note that the first term in \autoref{eqn:cost-expand}: $\mathbb{E}_{x_{1}, y_{1} \sim p(x, y| \theta^*) } \log q_\xi(y_{1}|x_{1})=\mathbb{E}_{D_1^*\sim p(D^*|\theta^*)}\log l_\xi(D_1^*)$ since the posterior distribution $q_\xi$ has not been conditioned on any context data.

We expand the first two terms in \autoref{eqn:cost-expand} (note $p(x,y|\theta^*)=p(x|\theta^*)p(y|x, \theta^*)$, and purely for notational brevity, we drop the conditional dependence of $x, y$ on $\theta^*$ below):
\begin{align}
	&\mathbb{E}_{x_{1}, y_{1} \sim p(x, y| \theta^*) } \log q_\xi(y_{1}|x_{1}) + \underbrace{\mathbb{E}_{x_{2}, y_{2}, x_{1}, y_{1} \sim p(x_2, y_2, D_1^*=(x_1, y_1)| \theta^*) } \log q_\xi(y_{2}|x_{2}, x_{1}, y_{1})}_{\mathbb{E}_2} \nonumber \\
	 =&\smallint_{y_1}\smallint_{x_1} \log q_\xi(y_1|x_1) p(x_1)p(y_1|x_1)dx_1 dy_1 + \mathbb{E}_2 \nonumber \\
	 =&\smallint_{y_1}\smallint_{x_1} \log q_\xi(y_1|x_1) p(x_1)p(y_1|x_1) \underbrace{\Big(\smallint_{y_2}\smallint_{x_2}p(x_2|x_1)p(y_2|x_2)dx_2 dy_2 \Big)}_{=1}dx_1 dy_1 + \mathbb{E}_2 \nonumber \\
	 =&\smallint_{y_1}\smallint_{x_1} \smallint_{y_2}\smallint_{x_2}\log q_\xi(y_1|x_1) p(x_1)p(y_1|x_1) p(x_2|x_1)p(y_2|x_2)dx_2 dy_2 dx_1 dy_1 + \mathbb{E}_2 \nonumber \\
	 =&\smallint_{x_2}\smallint_{y_2} \smallint_{x_1}\smallint_{y_1}\log \underbrace{q_\xi(y_1|x_1, x_2)}_{=q_\xi(y_1|x_1)} p(x_1)p(y_1|x_1) p(x_2|x_1)p(y_2|x_2)dy_1 dx_1 dy_2 dx_2 + \nonumber \\
	  &\smallint_{x_2}\smallint_{y_2} \smallint_{x_1}\smallint_{y_1}\log q_\xi(y_2|x_2, x_1, y_1)p(x_1)p(y_1|x_1)p(x_2|x_1)p(y_2|x_2)dy_1 dx_1 dy_2 dx_2 \nonumber \\
	 =&\smallint_{x_2}\smallint_{y_2} \smallint_{x_1}\smallint_{y_1}\log \big(q_\xi(y_1|x_1, x_2) q_\xi(y_2|x_2, x_1, y_1) \big)p(x_1, x_2)p(y_1|x_1)p(y_2|x_2)dy_1 dx_1 dy_2 dx_2 \nonumber \\
	 =&\smallint_{x_2}\smallint_{y_2} \smallint_{x_1}\smallint_{y_1}\log \big(q_\xi(y_1, y_2|x_1, x_2) \big)p(x_1, x_2)p(y_1, y_2|x_1, x_2)dy_1 dx_1 dy_2 dx_2 \nonumber \\
	 =&\smallint_{D_2^*}\log l_\xi(D_2^*) p(D_2^*|\theta^*)dD_2^*=\mathbb{E}_{D_2^*\sim p(D^*|\theta^*)}\log l_\xi(D_2^*)
\end{align}
Similarly by induction, we con conclude \autoref{eqn:cost-expand} evaluates to 
\begin{equation}
	\frac{1}{T} \min_\xi\mathbb{E}_{\theta^*\sim p(\theta)}[\mathbb{E}_{D_T^*\sim p(D^*|\theta^*)} \log l_\xi(D_T^*)]
\end{equation}
which is equivalent to the loss function (\autoref{eqn:mll-cost}) from GPR maximum likelihood method.
\end{proof}
\newpage
	
\section{Empirical Studies}
\label{sec:empirical}
In this section, we evaluate the performance of the above-mentioned Bayesian meta-learning frameworks with varying architectural and loss function differences.

\subsection{Environment Setup}
The source code used in the evaluations are based on implementations from \cite{alpaca} \footnote{ALPaCA:  \url{https://github.com/StanfordASL/ALPaCA}} and \cite{pacoh} \footnote{PACOH: \url{https://github.com/jonasrothfuss/meta_learning_pacoh}}.

Two synthetic dataset (Sinusoid, Cauchy) and one real-world dataset (Swissfel) are used for this benchmark. For sinusoid dataset, we generate training and test data from a family of sinusoid functions with added Gaussian noise: $y(x)=kx+A\sin(\omega x+b)+c+\epsilon$ with varying $k, A, \omega, b, c, \epsilon$ with their distribution summarized in \autoref{table:sinusoid-hyper}. Cauchy dataset is generated via the superposition of a common mean function (mixture of two Gaussians) across all tasks with GP prior functions sampled with an SE kernel. This synthetic dataset is of particular interest for evaluating the advantage of having a separate mean function estimator over sharing the same representation for both mean and kernel modules. Task data in Swissfel dataset are collected from multiple calibration sessions of Swiss Free Electron Laser (SwissFEL) (\cite{swissfel} \cite{swissfel-bayesian}). This dataset poses additional challenges such as potential overfitting due to the small amount of training tasks and samples available.

In this meta-learning setting, data from multiple tasks are included in the training dataset while test dataset consists of additional data from other unseen tasks. All data in training sets are used for optimizing bayesian prior hyper-parameters. Data in test sets are split into context data which are used to perform online update in BLR or GPR and test data which are evaluated for performance. To ensure a fair comparison between experiments, methods are trained and evaluated on the same training set and test set. \autoref{table:dataset-hyper} summarizes the dataset configurations.

For each experiment, AdamW (\cite{adamW}) optimizer are chosen over Adam (\cite{adam}) to avoid overfitting during training. The parameters of the optimizer such as learning rate and weight decay are tuned to achieve convergence in both training loss and validation loss. Lastly, we selected the best parameters based on performance in validation set (maximum log likelihood w.r.t posterior predictive distribution).

We evaluate the methods on three metrics: log likelihood of posterior predictive distribution, RMSE (root mean square error) and calibration error \cite{calibration} on test set. While RMSE captures the accuracy of mean prediction, calibration error reflects the accuracy in uncertainty estimates, namely the RMSE between empirical frequency and predicted frequency across different confidence intervals around mean prediction.

\subsection{Architecture and Loss Function}
In this section, we investigate the effect of different learning architecture and objective functions. The different configurations used for comparison can be found in \autoref{table:model-config}. For GP-based methods, we investigate the use of different kernel and mean function choices. `SE' refers to the use of a Squared Exponential Kernel directly on original data. `DSE' and `DL' correspond to the use of a Squared Exponential or Linear Kernel on a latent representation of raw data transformed by a deep neural network. An independent neural network is used for mean function prediction for `IN' models while the mean function in `SN' cases shares the same latent representation with its kernel.
\newpage

\begin{table}[t]
  \centering
  \begin{tabular}{lm{4.2cm}m{4.0cm}m{1.7cm}}
    \toprule
    Abbrev.     & Architecture    & Loss  & Reference\\
    \midrule
  	GPR-SE-IN  & GPR w/ SE kernel, independent deep mean function & \autoref{eqn:prior-f}: prior likelihood w/ full covariance  & \\ \hline
  	GPR-DSE-IN & GPR w/ deep SE kernel, independent deep mean function & \autoref{eqn:prior-f}: prior likelihood w/ full covariance & \cite{pacoh} (MAP)\\ \hline
  	GPR-DL-IN  & GPR w/ deep linear kernel, independent deep mean function & \autoref{eqn:prior-f}: prior likelihood w/ full covariance & \\ \hline
  	GPR-DL-SN  & GPR w/ deep linear kernel, shared deep mean function & \autoref{eqn:prior-f}: prior likelihood w/ full covariance & \\ \hline
  	BLR-PR-FC  & BLR & \autoref{eqn:prior-f}: prior likelihood w/ full covariance & \\ \hline
  	BLR-PR-DC  & BLR & \autoref{eqn:prior-nf}: prior likelihood w/ diagonal covariance & \\ \hline
  	BLR-POO-D/FC & BLR & \autoref{eqn:post-one}: posterior likelihood on $(x_{t+1}, y_{t+1})$ & \cite{alpaca} \\ \hline
  	BLR-POM-FC & BLR & \autoref{eqn:post-f}: posterior likelihood w/ full covariance on $(x_{t+1}, y_{t+1}, \dots ,x_{T}, y_{T})$ & \\ \hline
  	BLR-POM-DC & BLR & \autoref{eqn:post-nf}: posterior likelihood w/ diagonal covariance on $(x_{t+1}, y_{t+1}, \dots ,x_{T}, y_{T})$ & ALPaCA Code\\ \hline
    \bottomrule
  \end{tabular}
    \caption{Model Abbreviations and configuration details.}
    \label{table:model-config}
\end{table}

While the architecture choice for BLR is limited (equivalent to GPR-DL-SN as shown in \autoref{prop:arch-equivalence}), the equivalence between prior and posterior loss function (\autoref{prop:loss-equvilance}) is examined empirically. 

Note in the original ALPaCA paper, it is stated  that the objective (posterior log likelihood) is evaluated on the next available data point (\autoref{eqn:alpaca-cost}): 
\begin{equation}
\label{eqn:post-one}
	\mathbb{E}_{t\sim p(t), \theta^*\sim p(\theta)}\{\mathbb{E}_{x_{t+1}, y_{t+1}, D_t^* \sim p(x_{t+1}, y_{t+1}, D_t^* | \theta^*) } \log q_\xi(y_{t+1}|x_{t+1}, D_t^*)\}
\end{equation}
while the paper's accompanying source code implements it as:
\begin{equation}
\label{eqn:post-nf}
\mathbb{E}_{t\sim p(t), \theta^*\sim p(\theta)}\{\mathbb{E}_{\tau\sim \mathcal{U}(t+1, T)}[\mathbb{E}_{x_\tau, y_\tau, D_t^* \sim p(x_{t+1}, y_{t+1}, D_t^* | \theta^*) } \log q_\xi(y_{\tau}|x_{\tau}, D_t^*)]\}
\end{equation}
We propose a similar version which also takes into account the covariance between data points as follows:
\begin{equation}
\label{eqn:post-f}
	\mathbb{E}_{t\sim p(t), \theta^*\sim p(\theta)}\{\mathbb{E}_{X, Y, D_t^* \sim p(x_{t+1, \dots, T}, y_{t+1, \dots, T}, D_t^* | \theta^*) } \log q_\xi(Y|X, D_t^*)\}
\end{equation}
Similarly, for the prior likelihood cost (\autoref{eqn:mll-cost}) evaluated on a batch of data $D^*$:
\begin{equation}
\label{eqn:prior-f}
	\mathbb{E}_{\theta^*\sim p(\theta)}[\mathbb{E}_{D_T^*\sim p(D^*|\theta^*)} \log l_\xi(D_T^*)]
\end{equation}
we can derive the version evaluated on a single sample averaged over the dataset as follows:
\begin{equation}
\label{eqn:prior-nf}
	\mathbb{E}_{\theta^*\sim p(\theta)}\{\mathbb{E}_{\tau\sim \mathcal{U}(0, T)}[\mathbb{E}_{x_\tau, y_\tau \sim p(x_\tau, y_\tau|\theta^*)} \log l_\xi(x_\tau, y_\tau)]\}
\end{equation}

Starting from the model in PACOH-MAP(\cite{pacoh}), by changing one element at a time, we generated a total of 8 models with different architecture and loss function combination before we arrive at the standard ALPaCA (\cite{alpaca}) implementation. The configurations of these models are summarized in \autoref{table:model-config}. The performance of these methods across datasets are presented in \autoref{table:result} and \autoref{table:result-rmse}.
\newpage

\begin{table}[!ht]
\small
  \centering
  \begin{tabular}{l|>{\centering\arraybackslash}m{1.8cm}>{\centering\arraybackslash}m{1.8cm}>{\centering\arraybackslash}m{1.8cm}>{\centering\arraybackslash}m{1.8cm}}
    \toprule
   \multicolumn{1}{c}{Method} & \multicolumn{1}{c}{Sinusoid-Easy} & \multicolumn{1}{c}{Sinusoid-Hard} & \multicolumn{1}{c}{Cauchy} &\multicolumn{1}{c}{Swissfel} \\
    \midrule
    GPR-SE-IN	&0.313 &-0.112 & 0.394 & -0.447	\\
    GPR-DSE-IN	&0.596 &-0.348 & 0.185 & 0.763  \\
    GPR-DL-IN	&0.122 &-0.768 &	-0.015 &-1.228 \\
    GPR-DL-SN	&0.141 &-0.793 &0.016 &-0.645\\
    BLR-PR-FC	&-0.203 &-0.362	&0.011	&-0.826	\\
    BLR-PR-DC	&-1.210	 &-1.707	&-0.308	&-1.768	\\
    BLR-POO-D/FC&-0.450	&-0.256	&0.044	&-0.979	\\
    BLR-POM-FC	&-0.373	&-0.379	&-0.038	&-1.892	\\
    BLR-POM-DC	&-0.587	&-0.401	&-0.193	&-1.406\\
    \bottomrule
  \end{tabular}
  \caption{Comparison of Log Likelihood of GPR vs BLR meta-learning methods with different configurations across different datasets.}
  \label{table:result}
\end{table}

\begin{table}[h]
\small
  \centering
  \begin{tabular}{l|cc|cc|cc|cc}
    \toprule
    \multicolumn{1}{c}{}&\multicolumn{2}{c}{Sinusoid-Easy} & \multicolumn{2}{c}{Sinusoid-Hard} & \multicolumn{2}{c}{Cauchy} &\multicolumn{2}{c}{Swissfel}       \\
    \cmidrule(r){2-3} \cmidrule(r){4-5} \cmidrule(r){6-7} \cmidrule(r){8-9}
    Method & RMSE & Calib. & RMSE & Calib. & RMSE & Calib. & RMSE & Calib. \\
    \midrule
    GPR-SE-IN	&0.315 & 0.120 &0.644 &0.108 &0.200	&0.060	&0.368	&0.086 \\
    GPR-DSE-IN	&0.287 & 0.124 & 0.614 &0.104 &0.217&0.069	&0.443	&0.057 \\
    GPR-DL-IN	&0.248 & 0.130 &0.637	&0.105 &0.239 &0.074 &0.663	&0.076 \\
    GPR-DL-SN	&0.218 & 0.142 &0.644	&0.109 &0.230 &0.076 &0.459	&0.054\\
    BLR-PR-FC	&0.340	& 0.118 &0.591	&0.097	&0.225	&0.078	&0.479	&0.074\\
    BLR-PR-DC	&0.748	& 0.173 &0.878	&0.147	&0.237	&0.112	&0.641	&0.146\\
    BLR-POO-D/FC&0.438	&0.111	&0.643	&0.100	&0.231	&0.075	&0.630	&0.078\\
    BLR-POM-FC	&0.404	&0.116	&0.627	&0.104	&0.246	&0.080	&0.828	&0.139\\
    BLR-POM-DC	&0.481	&0.132	&0.725	&0.104	&0.234	&0.102	&0.967	&0.143\\
    \bottomrule
  \end{tabular}
  \caption{Comparison of RMSE and Calibration Error GPR vs BLR meta-learning methods with different configurations across different datasets.}
  \label{table:result-rmse}
\end{table}

\textbf{Architecture Equivalence} As discussed in \autoref{prop:arch-equivalence}, subjected to the same loss function, GPR-DL-SN and BLR-PR-FC in theory should yield the same posterior predictive distribution. Although GPR-DL-SN achieves slightly better performance in Sinusoid-Easy, we observe comparable results on other three datasets. The disparity in performance is most likely due to different optimization outcomes rather than representation power difference. Further benchmark on more dataset is required to determine whether one architecture has the advantage over the other from an optimization point of view.

\textbf{Prior and Posterior Loss Equivalence} It is argued in \autoref{prop:loss-equvilance} that PR-FC is mathematically equivalent to POO-D/FC. Across all four datasets, we observe a small advantage in RMSE of BLR-PR-FC over BLR-POO-D/FC. BLR-PR-FC and BLR-POO-D/FC are also consistently the best two performed models among the five loss choices in terms of log likelihood.

\textbf{Full vs Diagonal Covariance Loss} A significant drop in log likelihood is noticed after switching from PR-FC/POM-FC to PR-DC/POM-DC across all dataset, with the exception of Swissfel where both POM-FC and POM-DC fails. This suggests that capturing the correlation between data points in the loss function is beneficial to ensure good generalization, which is an important aspect of prior learning. 

\textbf{Posterior on One vs All Samples} BLR-POO-D/FC and BLR-POM-FC performs similarly except for Swissfel dataset. For i.i.d data with $p(x_t, y_t)=p(x_{t'}, y_{t'})$, \autoref{eqn:post-one} and \autoref{eqn:post-f} evaluates to the same quantity. Since the data in Swissfel corresponds to real data collected during a a calibration process which uses Bayesian optimization (\cite{swissfel-bayesian}), the non-i.i.d property of the training data may have attributed to this drop. 

\textbf{Shared vs Independent Latent Representation between Linear Kernel and Mean Function} In theory, having a separate neural network for mean function prediction should be more expressive and leads to better performance than the case with shared representation for both kernel feature space and mean function. However, sharing the same latent space might force it to learn a better representation which is meaningful both in terms of target variable and correlation between samples. Having a separate mean deep network did not yield improvement in Sinusoid Dataset, Cauchy Dataset and has led to poorer performance in Swissfel dataset. The benefit of having an independent mean module needs to be analyzed on a per-case basis.

\subsection{Approximation of Mercer Kernel with Neural Networks}
In \autoref{sec:kernel-equivalence}, we argue that any Mercer Kernel can be approximated by a linear kernel operating on a finite-dimensional latent space. To study the influence of network size of this latent representation on the quality of covariance prediction, we experiment BLR-PR-FC with different network sizes on Sinusoid-Hard dataset. 

\begin{figure}
	\centering
    \includegraphics[width=\textwidth]{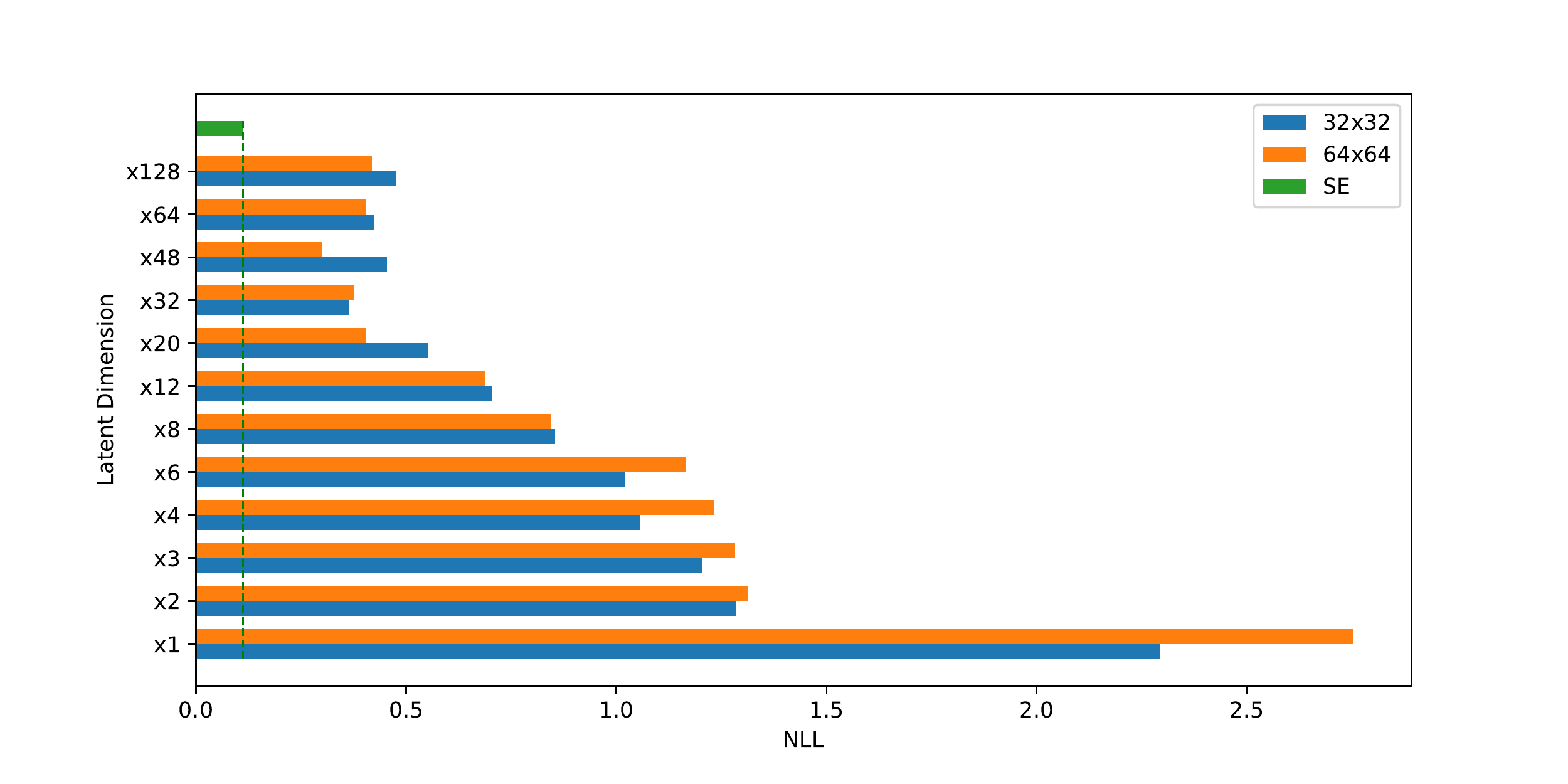}
    \caption{NLL(Negative Log Likelihood) on Sinusoid-Hard test set by BLR-PR-FC with varying latent space $\phi(x)$ size. As an example, the bottom orange bar corresponds to a $\phi(x)$ parameterized by a 64x64x1 network.}
    \label{fig:nn-size}
\end{figure}

As shown in \autoref{fig:nn-size}, while widening the feature dimension of the latent space, there is an initial trend in performance increase until the dimension reaches 32, after which the performance plateaus. We can observe a shift of bottleneck from the representation power of latent representation to the insufficient data available for further exploitation. The initial improvement is also reflected in the posterior predictive distribution shown in \autoref{fig:alpaca-p} and \autoref{fig:alpaca-g}.

Empirically, it is apparent this learned latent space representation is not as effective as a common kernel such as SE kernel. Among all candidates, GPR-SE-IN and GPR-DSE-IN are the best performing models across all dataset with significant margins compared to GPR-DL-IN as shown in \autoref{table:result}.

In conclusion, it is helpful to adopt a hand-crafted kernel, which already measures meaningful similarity, either on raw data or transformed latent space, if necessary. It is highly unlikely, with the limited amount of training data, that a latent feature representation of a powerful kernel could be directly learned.

\begin{figure}[!htb]
\centering
  \begin{minipage}{0.47\textwidth}
  	\centering
    \includegraphics[width=\textwidth]{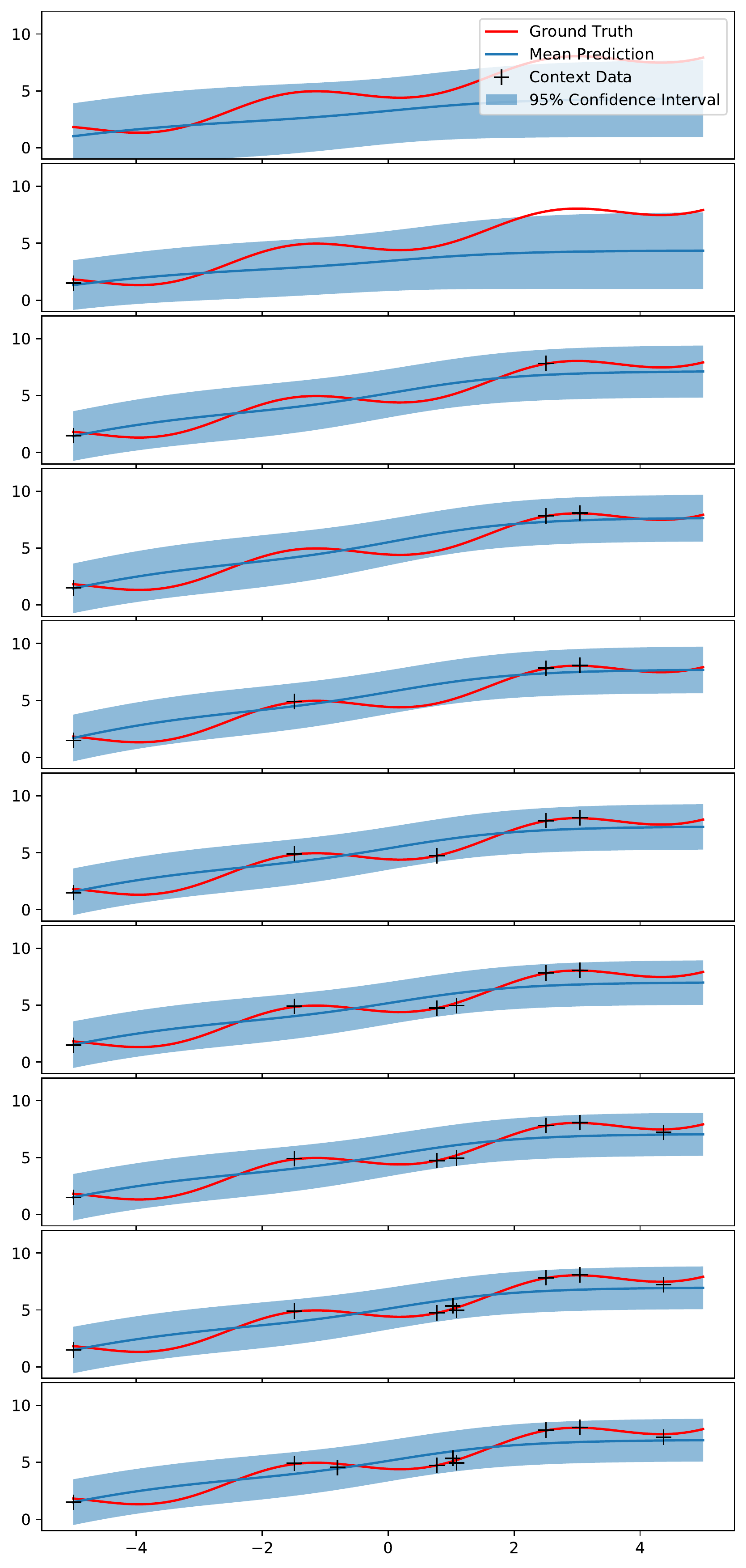}
    \caption{Posterior Prediction by BLR-PR-FC with $\phi(x)$ of size 32x32x3 on Sinusoid-Hard test set.}
    \label{fig:alpaca-p}
  \end{minipage}
  \begin{minipage}{0.47\textwidth}
  	\centering
    \includegraphics[width=\textwidth]{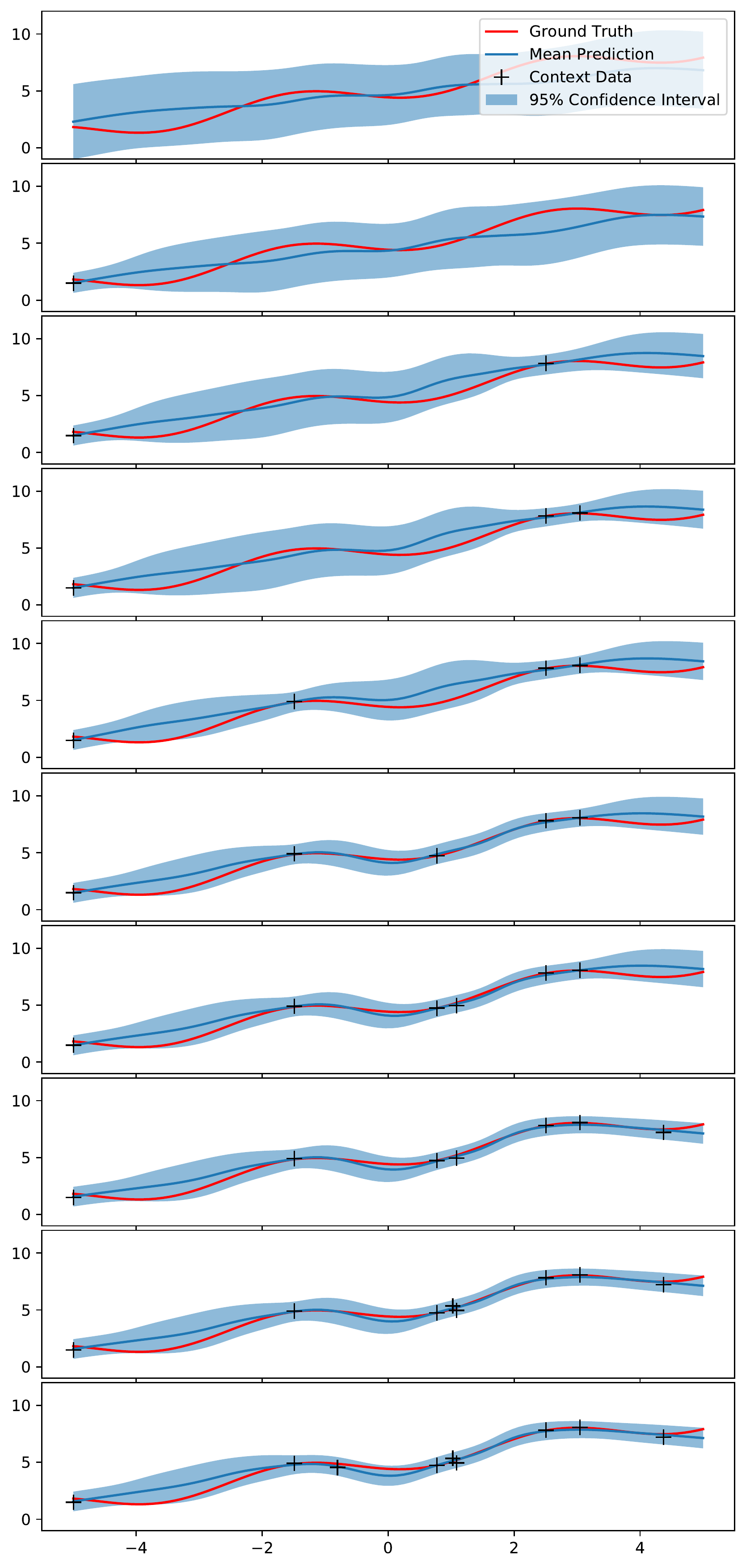}
    \caption{Posterior Prediction by BLR-PR-FC with $\phi(x)$ of size 32x32x32 on Sinusoid-Hard test set.}
    \label{fig:alpaca-g}
  \end{minipage}
\end{figure}

\newpage

\section{Conclusion}
In this report, we compared ALPaCA vs. PACOH-MAP, two recently published bayesian meta-learning frameworks and their variations. We showed theoretical equivalence in their model architecture, namely the equivalence between BLR (Bayesian Linear Regression) and GPR (Gaussian Process Regression). We extended this equivalence to the range of GPs with any Mercer kernel. We further showed the objective function used in ALPaCA and PACOH are also equivalent under mild assumptions. Further empirical studies verified this architectural and loss function equivalence. Furthermore, other model derivatives with similar architecture and loss function revealed the importance of considering full covariance when evaluating multiple samples and the benefit of adopting an `already-learned' kernel in boosting performance with limited data.

We hope to continue investigating the comparison in the reinforcement learning setting where both training and test data are inherently non i.i.d. We believe it is also of interest to explore whether kernel approximation could be separated from the meta-learning process in that we first train a neural network to approximate a common kernel and use the fixed network in BLR to achieve fast update without compromise on performance.
\bibliography{cite.bib}
\bibliographystyle{plainnat}

\appendix
\section{Proof of Equivalence on Posterior Distribution of GPR and BLR}
\label{sec:proof}
In this section, we prove that the distribution of \autoref{eqn:alpaca-post} and \autoref{eqn:gpr-post} are the same.

For the Gaussian distribution mean, starting from \autoref{eqn:gpr-post}:
\begin{align}
	\mu_{GPR}=&\Phi(X_t)^\top \mathbf{K_0}+\Phi(X_t)^\top \Lambda_0^{-1}\Phi(X_c)(\Phi(X_c)^\top\Lambda_0^{-1}\Phi(X_c)+\mathbf{I})^{-1}(Y_c-\Phi(X_c)^\top \mathbf{K_0} \\
	=&\Phi(X_t)^\top [\mathbf{K_0}+ \Lambda_0^{-1}\Phi(X_c)(\Phi(X_c)^\top\Lambda_0^{-1}\Phi(X_c)+\mathbf{I})^{-1}(Y_c-\Phi(X_c)^\top \mathbf{K_0})] \\
	=&\Phi(X_t)^\top [\mathbf{I} - \Lambda_0^{-1}\Phi(X_c)(\Phi(X_c)^\top\Lambda_0^{-1}\Phi(X_c)+\mathbf{I})^{-1}\Phi(X_c)^\top]\mathbf{K_0}  \nonumber \\
	\label{eqn:16}
	&+\Phi(X_t)^\top [\Lambda_0^{-1}\Phi(X_c)(\Phi(X_c)^\top\Lambda_0^{-1}\Phi(X_c)+\mathbf{I})^{-1}] Y_c
\end{align}
Note that by applying Woodbury identity \footnote{$(\mathbf{A} + \mathbf{C}\mathbf{B}\mathbf{C}^\top)^{-1}=\mathbf{A}^{-1}+\mathbf{A}^{-1}\mathbf{C}(\mathbf{B}^{-1}+\mathbf{C}^\top\mathbf{A}^{-1}\mathbf{C})^{-1}\mathbf{C}^\top \mathbf{A}^{-1}$} on \autoref{eqn:lambda-tau}:
\begin{align}
	\label{eqn:17}
	\Lambda_\tau^{-1}&=(\Lambda_0+\Phi(X_c)\mathbf{I}\Phi(X_c)^\top)^{-1} \nonumber \\
	&=\Lambda_0^{-1}-\Lambda_0^{-1}\Phi(X_c)(\mathbf{I}+\Phi(X_c)^\top\Lambda_0^{-1}\Phi(X_c))^{-1}\Phi(X_c)^\top\Lambda_0^{-1} \\
	\label{eqn:18}
	\Lambda_\tau^{-1}\Lambda_0&=\mathbf{I}-\Lambda_0^{-1}\Phi(X_c)(\mathbf{I}+\Phi(X_c)^\top\Lambda_0^{-1}\Phi(X_c))^{-1}\Phi(X_c)^\top
\end{align}
In addition,
\begin{align}
	\Lambda_\tau \Lambda_0^{-1}\Phi(X_c)&=(\Phi(X_c) {\Phi(X_c)}^\top+\Lambda_0)\Lambda_0^{-1}\Phi(X_c) \\
	&=\Phi(X_c)\Phi(X_c)^\top\Lambda_0^{-1}\Phi(X_c)+\Phi(X_c) \\
	&=\Phi(X_c)(\mathbf{I}+\Phi(X_c)^\top\Lambda_0^{-1}\Phi(X_c))
\end{align}
Multiply both sides by $\Lambda_\tau^{-1}$,
\begin{align}
	\Lambda_0^{-1}\Phi(X_c)&=\Lambda_\tau^{-1}\Phi(X_c)(\mathbf{I}+\Phi(X_c)^\top\Lambda_0^{-1}\Phi(X_c)) \\
	\label{eqn:23}
	\Lambda_0^{-1}\Phi(X_c)(\mathbf{I}+\Phi(X_c)^\top\Lambda_0^{-1}\Phi(X_c))^{-1}&=\Lambda_\tau^{-1}\Phi(X_c)
\end{align}
Substitute \autoref{eqn:18} and \autoref{eqn:23} into \autoref{eqn:16},
\begin{align}
	\mu_{GPR}&=\Phi(X_t)^\top \Lambda_\tau^{-1}\Lambda_0\mathbf{K_0}+\Phi(X_t)^\top \Lambda_\tau^{-1}\Phi(X_c)Y_c\\
	&=\Phi(X_t)^\top\Lambda_\tau^{-1}(\Lambda_0\mathbf{K_0}+\Phi(X_c)Y_c) \\
	&=\mu_{BLR}
\end{align}

For the covariance matrix, we start from \autoref{eqn:gpr-post}:
\begin{align}
	&\Sigma_{GPR} \nonumber \\
	=&[\Phi(X_t)^\top\Lambda_0^{-1}\Phi(X_t)+\mathbf{I}-\Phi(X_t)^\top\Lambda_0^{-1}\Phi(X_c)(\Phi(X_c)^\top\Lambda_0^{-1}\Phi(X_c)+\mathbf{I})^{-1}(\Phi(X_c)^\top\Lambda_0^{-1}\Phi(X_t))]\nonumber \\
	&\otimes \Sigma_\epsilon \nonumber \\
	=&\{\mathbf{I}+\Phi(X_t)^\top[\Lambda_0^{-1}-\Lambda_0^{-1}\Phi(X_c)(\Phi(X_c)^\top\Lambda_0^{-1}\Phi(X_c)+\mathbf{I})^{-1}\Phi(X_c)^\top\Lambda_0^{-1}]\Phi(X_t)\} \otimes \Sigma_\epsilon
\end{align}
Subsitute \autoref{eqn:17}, we get:
\begin{align}
	\Sigma_{GPR}=(\mathbf{I}+\Phi(X_t)^\top\Lambda_\tau^{-1}\Phi(X_t)) \otimes \Sigma_\epsilon=\Sigma_{BLR}
\end{align}

\section{Proof of Equivalent Posterior Distribution with Linearly Transformed $\mathbf{K}$ Prior Distribution under Latent Bayesian Linear Regression}
\label{sec:proof-blr}
Since both $\Lambda_0$ and $\Lambda_0^{-1}$ are symmetric positive-definite matrices, we can perform Cholesky decomposition on them: $\Lambda_0^{-1}=\mathbf{L^\top} \mathbf{L}$. 

We claim the posterior predictive distribution in \autoref{eqn:alpaca-post} is equivalent after a linear transformation on $\Lambda_0^{\prime -1}=\mathbf{L}^{\prime \top} \Lambda_0^{-1} \mathbf{L'}$, $\phi^{\prime}(x)=\mathbf{L}^{\prime -1}\phi(x)$ and $\mathbf{K_0^\prime}=\mathbf{L}^{\prime\top}\mathbf{K_0}$ : 
\begin{align*}
	p(Y_t|X_t, X_c, Y_c)&=\mathcal{N}(\Phi(X_t)^\top\mathbf{K_\tau}, \Sigma_\tau) \\
						&=\mathcal{N}(\Phi'(X_t)^\top\mathbf{K_\tau^\prime}, \Sigma^\prime_\tau)
\end{align*}

First note,
\begin{align}
	\Phi'(X_t)^\top &=\Phi(X_t)^\top \mathbf{L'^{-\top}} \\
	\Lambda_\tau^{\prime -1} &= (\Phi'(X_c) {\Phi'(X_c)}^\top+\Lambda_0^\prime)^{-1}\\
	&= (\mathbf{L'}^{-1}\Phi(X_c) \Phi(X_c)^\top\mathbf{L'}^{-\top}+\mathbf{L'}^{-1}\Lambda_0\mathbf{L'}^{-\top})^{-1}\\
	&= \big[\mathbf{L'}^{-1}( \Phi(X_c) \Phi(X_c)^\top+\Lambda_0)\mathbf{L'}^{-\top}\big]^{-1}=\big[\mathbf{L'}^{-1}\Lambda_\tau\mathbf{L'}^{-\top}\big]^{-1}\\
	&= \mathbf{L'}^\top \Lambda_\tau^{-1} \mathbf{L'}  \\
	\mathbf{K_\tau^\prime} &= \Lambda_\tau^{\prime -1}(\Phi'(X_c) Y_c + \Lambda_0^\prime\mathbf{K_0^\prime}) \\
	&=\Lambda_\tau^{\prime -1}(\mathbf{L'}^{-1}\Phi(X_c)Y_c+\mathbf{L'}^{-1}\Lambda_0\mathbf{L'}^{-\top}\mathbf{L}^{\prime\top}\mathbf{K_0}) \\
	&=\mathbf{L'}^\top \Lambda_\tau^{-1} \mathbf{L'} \mathbf{L'}^{-1}(\Phi(X_c)Y_c+\Lambda_0\mathbf{K_0}) \\
	&=\mathbf{L'}^\top \Lambda_\tau^{-1}(\Phi(X_c)Y_c+\Lambda_0\mathbf{K_0}) = \mathbf{L'}^\top \mathbf{K_\tau} 
\end{align}
Apply the above for the distribution mean and covariance:
\begin{align}
	\Phi'(X_t)^\top\mathbf{K_\tau^\prime}&=\Phi(X_t)^\top \mathbf{L'^{-\top}}\mathbf{L'}^\top \mathbf{K_\tau}=\Phi(X_t)^\top \mathbf{K_\tau} \\
	\Sigma^\prime_\tau &=(\mathbf{I}+\Phi'(X_t)^\top \Lambda_\tau^{\prime -1} \Phi'(X_t))\otimes \Sigma_\epsilon \\
	&=(\mathbf{I}+\Phi(X_t)^\top \mathbf{L'}^{-\top} \mathbf{L'}^\top \Lambda_\tau^{-1} \mathbf{L'}  \mathbf{L}^{\prime -1} \Phi(X_t))\otimes \Sigma_\epsilon \\
	&=(\mathbf{I}+\Phi(X_t)^\top \Lambda_\tau^{-1} \Phi(X_t))\otimes \Sigma_\epsilon=\Sigma_\tau
\end{align}

\section{Dataset Details}
The following hyper-parameters are used for generating and evaluation of the datasets:
\begin{table}[h!]
  \centering
  \begin{tabular}{lccccc}
    \toprule
    &\multicolumn{2}{c}{Training Set} &\multicolumn{3}{c}{Test Set}  \\
    \cmidrule(r){2-3} \cmidrule(r){4-6}
    Dataset     & \# Training tasks     & \# Samples & \# Test tasks & \# Context samples &  \# Test samples\\
    \midrule
    Sinusoid-Easy & 20  & 5 & 100 & 5 & 100     \\
    Sinusoid-Hard & 20  & 10 & 100 & 10 & 100  \\
    Cauchy     & 20 & 20 & 1000 & 20 & 100  \\
    Swissfel   & 5 & 200 & 4 & 200 & 200  \\
    \bottomrule
  \end{tabular}
    \caption{Hyper-parameters used in synthetic dataset generation and evaluation. Number of samples are per-task.}
    \label{table:dataset-hyper}
\end{table}

\subsection{Sinusoid-Easy and Sinusoid-Hard}
Both dataset are generated by uniformly sampling the family of sinusoid functions: $y(x)=kx+A\sin[\omega (x-b)]+c+\epsilon$ from $X=[-5, 5]$. The probability distribution of $k, A, \omega, b, c, \epsilon$ are as follows:

\begin{table}[h!]
  \centering
  \begin{tabular}{ccc}
    \toprule
    Parameter     & Sinusoid-Easy    &  Sinusoid-Hard \\
    \midrule
    $k$ & $\mathcal{N}(0.5, 0.2^2)$  &  $\mathcal{N}(0.5, 0.6^2)$ \\
    $A$ & $\mathcal{U}(0.7, 1.3)$  & $\mathcal{U}(0.7, 1.4)$  \\
    $\omega$ & 1.5  & $\mathcal{U}(1.0, 2.0)$   \\
    $b$ & $\mathcal{N}(0.1, 0.1^2)$  & $\mathcal{N}(0.0, 2.0^2)$  \\
    $c$ & $\mathcal{N}(5.0, 0.1^2)$ & $\mathcal{N}(5.0, 0.8^2)$ \\
    $\epsilon$ & $\mathcal{N}(0.0, 0.1^2)$ & $\mathcal{N}(0.0, 0.2^2)$\\
    \bottomrule
  \end{tabular}
    \caption{Sinusoid synthetic datasets configuration parameters.}
    \label{table:sinusoid-hyper}
\end{table}

\begin{figure}[!htb]
\centering
  \begin{minipage}{0.45\textwidth}
  	\centering
    \includegraphics[width=\textwidth]{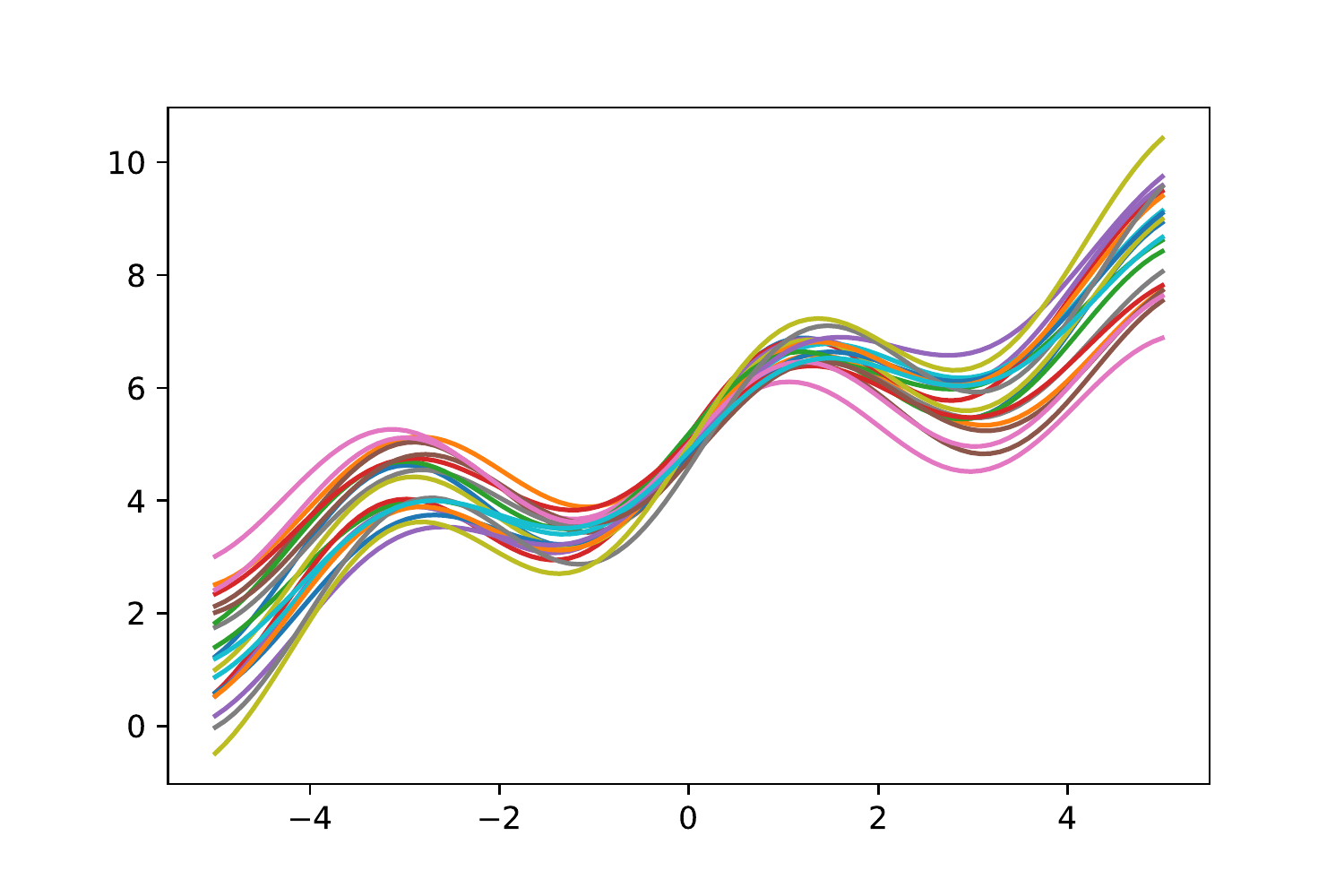}
    \caption{20 test functions drawn from Sinusoid-Easy test set.}
    \label{fig:1}
  \end{minipage}
  \begin{minipage}{0.45\textwidth}
  	\centering
    \includegraphics[width=\textwidth]{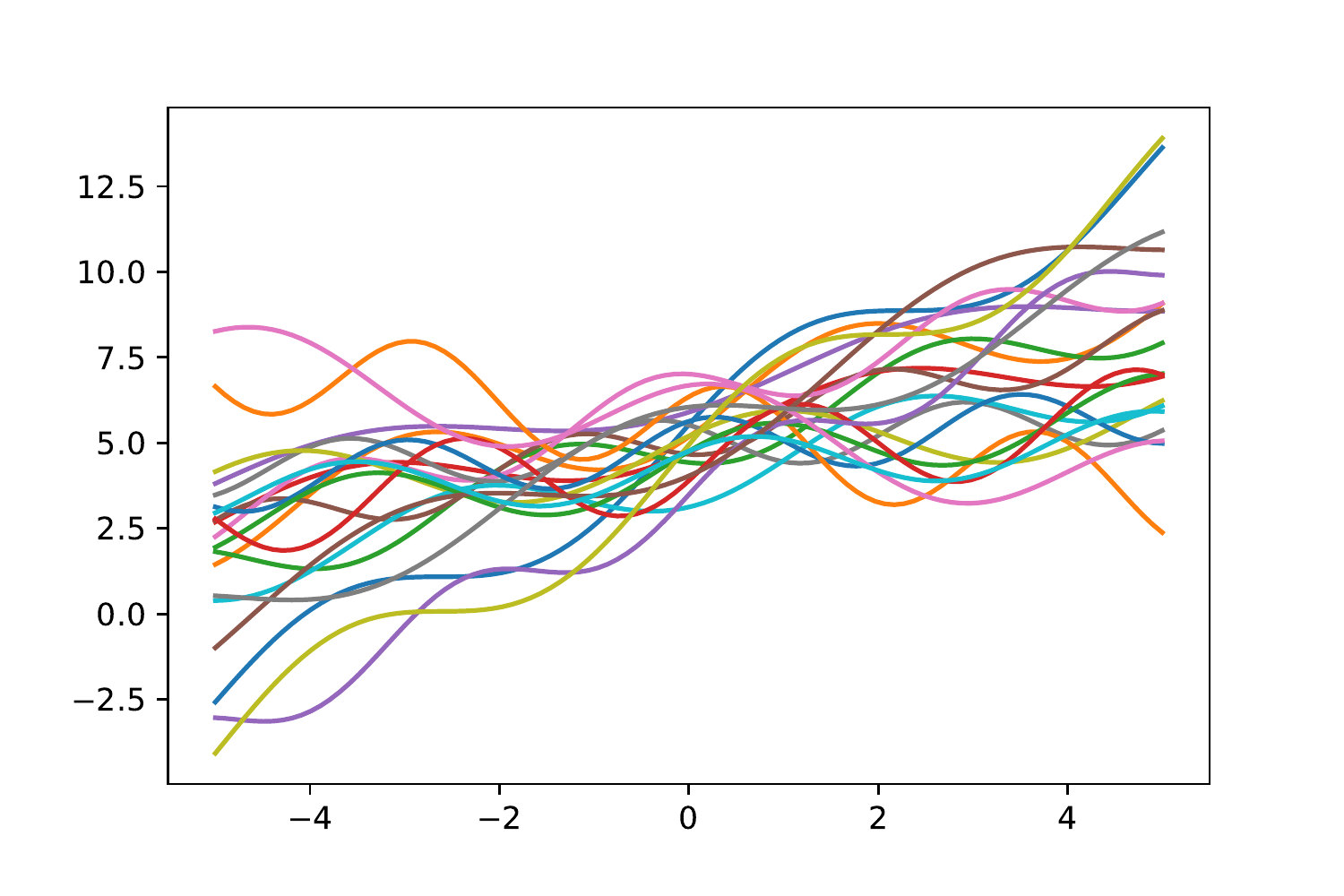}
    \caption{20 test functions drawn from Sinusoid-Hard test set.}
    \label{fig:2}
  \end{minipage}
\end{figure}


\subsection{Cauchy \& Swissfel}
Please refer to Appendix C.1.2. and C.1.3. in \cite{pacoh} for more information.
\end{document}